\documentclass{article}
\usepackage{fullpage,times}
\usepackage{graphicx} 
\usepackage{authblk} 
\usepackage{subfigure, floatrow}

\usepackage{amsfonts}
\usepackage{url}
\usepackage{multirow}
\usepackage{algorithm,algorithmic,amsmath,amsthm,amsfonts,bbm}
\usepackage{verbatim,graphicx,subfigure}
\usepackage{mathtools}
\newcommand{\BlackBox}{\rule{1.5ex}{1.5ex}}  
 
\newtheorem*{assumption*}{Assumption}
 
\newtheorem{thm}{Theorem}

\newtheorem{lemma}{Lemma} 
\newtheorem{proposition}{Proposition}

\def\bbE{\mathbb{E}}

\def\bbR{\mathbb{R}}

\def\cW{\mathcal{W}}
\def\cH{\mathcal{H}}

\def\cF{\mathcal{F}}
\def\cG{\mathcal{G}}

\newcommand{\defeq}{\mbox{$\;\stackrel{\mbox{\tiny\rm def}}{=}\;$}}

\newcommand{\leqa}{\mbox{$\;\stackrel{\mbox{\tiny\rm (a)}}{\leq}\;$}}
\newcommand{\leqb}{\mbox{$\;\stackrel{\mbox{\tiny\rm (b)}}{\leq}\;$}}

\DeclareMathOperator{\err}{err}

\DeclarePairedDelimiter\abs{\lvert}{\rvert}%
\DeclarePairedDelimiter\norm{\lVert}{\rVert}%

\makeatletter
\let\oldabs\abs
\def\abs{\@ifstar{\oldabs}{\oldabs*}}
\let\oldnorm\norm
\def\norm{\@ifstar{\oldnorm}{\oldnorm*}}
\makeatother

\usepackage{todonotes,bm, amsmath, amssymb}

\usepackage{algorithm}
\usepackage{algorithmic}

\usepackage{hyperref}

\newfloatcommand{capbtabbox}{table}[][\FBwidth]

\newcommand{\cN}{\mathcal{N}}
\newcommand{\R}{\mathbb{R}}

\newcommand{\cE}{\mathcal{E}}

\newcommand{\vct}[1]{\bm{#1}}
\newcommand{\mtx}[1]{\bm{#1}}
\newcommand{\vx}{\vct{x}}
\newcommand{\ve}{\vct{e}}

\newcommand{\vy}{\vct{y}}
\newcommand{\vz}{\vct{z}}

\newcommand{\vw}{\vct{w}}
\newcommand{\vwstar}{\vct{w}_{\star}}
\newcommand{\vwhat}{\hat{\vw}}
\newcommand{\vbeta}{\vct{\beta}}

\newcommand{\vp}{\vct{p}}
\newcommand{\vq}{\vct{q}}

\newcommand{\mX}{\mtx{X}}

\newcommand{\T}{\mathcal{T}}

\newcommand{\gstar}{g_{\star}}
\newcommand{\Phistar}{\Phi_{\star}}

\newcommand{\hatu}{\hat{g}}
\newcommand{\hath}{\hat{h}}

\newcommand{\wstar}{\vct{w}_{*}}

\newcommand{\what}{\hat{\vw}}
\newcommand{\ghat}{\hat{g}}
\newcommand{\bi}{\{ i \}}
\def\T0{T_0}

\newcommand{\silo}{\text{SILO }}
\newcommand{\isilo}{\text{i\silo}}
\newcommand{\cisilo}{\text{ci\silo}}
\newcommand{\hsim}{\text{\silo}}

\def\reals{{\mathbb R}}

\def\expect{{\mathbb E}}

\newcommand{\beq}{\begin{equation}}
\newcommand{\eeq}{\end{equation}}

\newcommand{\opterr}{\operatorname{opterr}}

\begin{document}
\title{Learning Single Index Models in High Dimensions}
\author[1]{Ravi Ganti \thanks{gantimahapat@wisc.edu}}
\author[2]{Nikhil Rao \thanks{nikhilr@cs.utexas.edu}}
\author[3]{Rebecca M. Willett \thanks{rmwillett@wisc.edu}}
\author[3]{Robert Nowak\thanks{rdnowak@wisc.edu}}
\affil[1]{Wisconsin Institutes for Discovery, 330 N Orchard St, Madison, WI, 53715}
\affil[2]{Department of Computer Science, University of Texas at Austin, 78712}
\affil[3]{Department of Electrical and Computer Engineering, University of Wisconsin-Madison, Madison, WI, 53706}
\renewcommand\Authands{ and }
\date{}
\maketitle

\begin{abstract} 
Single Index Models (SIMs) are simple yet flexible semi-parametric models for classification and regression. Response variables are modeled as a nonlinear, monotonic function of a linear combination of features. Estimation in this context requires learning both the feature weights, and the nonlinear function. While methods have been described to learn SIMs in the low dimensional regime, a method that can efficiently learn SIMs in high dimensions has not been forthcoming. We propose three variants of a  computationally and statistically efficient algorithm for SIM inference in high dimensions. We establish excess risk bounds for the proposed algorithms and experimentally validate the advantages that our SIM learning methods provide relative to Generalized Linear Model (GLM) and low dimensional SIM based learning methods. 
%
%
%
\end{abstract}


\section{Introduction}
\label{sec:intro}

High-dimensional learning is often tackled using generalized linear models,
where we assume that a response variable $Y \in \reals$ is related to a feature
vector $X \in \reals^d$ via 
\beq
\label{eq:sim}
\expect[Y|X=\vx] = \gstar(\vwstar^\top\vx)
\eeq
for some weight vector $\vwstar \in \reals^d$ and some monotonic and smooth function $\gstar$ called the transfer function. Typical examples of $\gstar$ are the logit function and the probit function for classification, and the linear function for regression. While classical
work on generalized linear models (GLMs) assumes $\gstar$ is known, this
potentially nonlinear function is often unknown and hence a major challenge in
statical inference.

%

The model in \eqref{eq:sim} with $\gstar$ unknown is called a {\em
  Single Index Model (SIM)} and is a powerful
semi-parametric generalization of a GLM~. SIMs were first introduced in econometrics and statistics~\cite{horowitz1996direct,ichimura1993semiparametric, horowitz2009semiparametric}. 
Recently,
computationally and statistically efficient algorithms have been
provided for learning SIMs~\cite{sim_ravi,sim_sham} in low-dimensional
settings where the number of samples/observations $n$ is larger than the ambient dimension $d$. However, modern data analysis problems in
machine learning, signal processing, and computational biology involve
high dimensional datasets, where the number of parameters far
exceeds the number of samples ($n \ll d$). 

In this paper we consider the problem of learning SIMs, given labeled data, in the \emph{high-dimensional} regime. We provide algorithms that are both computationally and statistically efficient for learning SIMs in high-dimensions, and validate our methods on several high dimensional datasets. Our contributions in this paper can be  summarized as follows:
\begin{enumerate}
\item We propose a suite of algorithms to learn SIMs in high dimensions. Our simplest algorithm called $\silo$ (Single Index Lasso Optimization) is a simple, non iterative method that estimates the vector $\vwstar$ and a monotonic, Lipschitz function $\gstar$. \isilo and \cisilo are iterative variants of \silo that use different loss functions. While \isilo uses a squared loss function, \cisilo uses a calibrated loss function that adapts to the SIM from which our data is generated.
\item We provide excess risk bounds on the hypotheses returned by SILO, iSILO, ciSILO. 
\item We experimentally compare our algorithms with other methods used both for SIM learning and high dimensional parameter estimation on various real world high dimensional datasets. Our experimental results show superior performance of \isilo and \cisilo when compared to commonly used methods for high dimensional estimation. 
\end{enumerate}

The rest of the paper is organized as follows: In Section~\eqref{sec:ma}, we formally set up the problem we wish to solve, and detail the proposed methods, SILO, iSILO, ciSILO. In Section \eqref{sec:theory}, we perform a theoretical analysis of SILO, iSILO, and \cisilo. We perform a thorough empirical evaluation on several datasets in Section \eqref{sec:exp}, and conclude our paper in Section \eqref{sec:conc}. Full proofs of our theoretical analysis are available in the appendix.


\subsection{Related work}
\label{sec:rw}
High dimensional parameter estimation for GLMs has been
widely studied, both from a theoretical and algorithmic point of view ( \cite{van2008high, Mest, park2007l1} and references therein). Learning SIMs  is a harder problem and was first introduced in econometrics~\cite{ichimura1993semiparametric} and statistics ~\cite{horowitz1996direct}. In~\cite{sim_ravi} the authors proposed and analyzed the Isotron algorithm to learn SIMs in the low dimensional setting. Isotron uses perceptron type updates to learn $\vwstar$, along with application of the Pool Adjacent Violator (PAV) algorithm to learn $\gstar$. This was improved in~\cite{sim_sham} where the authors proposed the Slisotron algorithm that combined perceptron updates to learn $\vwstar$ along with a Lipschitz PAV (LPAV) procedure to learn $\gstar$. Both the Isotron and the Slisotron algorithm rely on performing perceptron updates. While the perceptron algorithm works for low-dimensional classification problems, to the best of our knowledge the performance of the perceptron algorithm has not been studied in high-dimensions. Hence, it is not clear if the Isotron and the Slisotron algorithms designed for learning SIM in low-dimensions would work in the high dimensional setting. 

Alquier and Biau~\cite{alquier2013sparse} consider learning high dimensional single index models. The authors provide estimators of $\gstar,\vwstar$ using PAC-Bayesian analysis. However, the estimator relies on reversible jump MCMC, and it is seemingly hard to implement. Also, the MCMC step is slow to converge even for moderately sized problems. To the best of our knowledge, simple, practical algorithms with theoretical guarantees and good empirical performance for learning single index models in high dimensions are not available.  Restricted versions of the SIM estimation problem have been considered in \cite{plan2014high, rao2014classification}, where  the authors are only interested in accurate parameter estimation and not prediction. Hence, in these works the proposed algorithms do not learn the transfer function.

\paragraph{The LPAV: }
Before we discuss algorithms for learning high dimensional SIMs, we discuss the LPAV algorithm proposed in \cite{sim_sham}, as an extension to the PAV method used in \cite{sim_ravi}. Given data $(p_1,y_1),\ldots (p_n,y_n)$, where $p_1,\ldots,p_n\in \bbR$ the LPAV outputs the best univariate monotonic, 1-Lipschitz function $\hat{g}$, that minimizes squared error $\sum_{i=1}^n (g(p_i)-y_i)^2$. In order to do this, the LPAV first solves the following optimization problem: 
\begin{equation}
\label{opt:lpav}
  \hat{\vz}=\arg\min_{\vz\in\bbR^n} ~ \|\vz-\vy\|_2^2 \quad \textbf{s.t.} ~\ 0 \leq z_j-z_i\leq p_j-p_i ~\text{if~} p_i \leq p_j
  \end{equation}
where $\hat{g}(p_i)=\hat{z}_i$. This gives us the value of $\hat{g}$ on a discrete set of points $p_1,\ldots,p_n$. To get $\hat{g}$ everywhere else on the real line, we simply perform linear interpolation as follows: Sort $p_i$  for all $i$ and let  $p_{\bi}$ be the $i^{th}$ entry after sorting. Then, for any $\zeta\in \bbR$, we have
\begin{equation}
   \label{eqn:interpolate}
    \hat{g}(\zeta)=
    \begin{cases}
     \hat{z}_{\{1\}}, & \text{if}\ \zeta \leq p_{\{ 1 \}} \\
      \hat{z}_{\{n\} }, & \text{if}\ \zeta \geq p_{\{ n \}} \\
      \mu \hat{z}_{ \{i \} } + (1 - \mu) \hat{z}_{\{i+1\}} & \text{if}\ \zeta = \mu p_{\bi} + (1-\mu) p_{\{ i+1\} }
    \end{cases}
  \end{equation} 
  In the algorithms that we shall discuss in this paper we shall invoke the LPAV routine with $p_i$ set to the projection of the data point $\vx_i$ on some algorithm-dependent weight vector $\vw$.


\section{Statistical model and proposed algorithms}
\label{sec:ma}
Assume we are provided i.i.d. data
$\{(\vx_1,y_1),\ldots,(\vx_n,y_n)\}$, where the label $Y$ is generated
according to the model $\bbE[Y|X=x] = \gstar(\vwstar^\top \vx)$ for an unknown
parameter vector $\vwstar \in \reals^d ~\ n \ll d$ and unknown 1-Lipschitz, monotonic function
$g_\star$. We additionally assume that $y\in [0,1]$, $\|\vwstar\|_2\leq 1$ and $\|\vwstar\|_0 \leq s$, where $\| \cdot \|_0$ is the $\ell_0$ pseudo-norm. The sparsity
assumption on $\vwstar$ is motivated by the fact that consistent
estimation in high dimensions is an ill-posed problem without making
 further structural assumptions on the underlying parameters. 
 
Our goal is to make predictions on unseen data. Specifically, we would like to provide estimators $\hat{g}$ and $\hat{\vw}$ of $\gstar$ and $\vwstar$ so that given a previously unseen sample $\vx$, we predict $\hat{\vy} = \hat{g}(\hat{\vw}^\top \vx)$. To this end, we propose three algorithms that we explain next

\subsection{SILO: Single Index Lasso Optimization}
\label{sec:silo}
We first propose SILO, a simple SIM learning algorithm that first learns $\hat{\vw}$ and then fits a function $\hat{g}$ using $\hat{\vw}$. Specifically, SILO performs the following two steps in a single pass:
\begin{enumerate}
\item In order to learn $\hat{\vw}$ we solve the problem that was first proposed in \cite{plan_1bit}. This optimization problem is independent of the transfer function $\gstar$ and minimizes a linear loss subject to model constraints:
\begin{equation}
\label{pvproblem}
\hat{\vw} = \arg \min_{\substack{\vw : \| \vw \|_2 \leq 1,\\ \| \vw \|_1 \leq \sqrt{s}}}  - \frac{1}{n} \sum_{i = 1}^n \vy_i \vx_i^\top \vw. 
\end{equation}
where the constraint $\|\vw\|_1\leq \sqrt{s}$ arises from constraining an $s-$sparse vector to have unit Euclidean norm.
\item After learning $\hat{\vw}$, \silo  simply fits a 1-Lipschitz monotonic function by invoking the LPAV routine with the vector $\vp=[p_1,\ldots,p_n]$, where $p_i=\hat{\vw}^\top\vx_i$. LPAV outputs a function $\hat{g}$. Our final predictor has the form $\hat{y}=\hat{g}(\hat{\vw}^\top\vx)$.
\end{enumerate}
Note that there is no need to re-learn $\vwhat$ after learning $\ghat$, since the optimization problem to learn $\vwhat$ is independent of $\ghat$. This property makes $\silo$ a very simple and a  computationally attractive algorithm. 

\subsection{iSILO: Iterative SILO with squared loss}
SILO is computationally very efficient, since it only involves learning $\hat{\vw}, \hat{g}$ once. However, completely ignoring $\hat{g}$ to learn $\hat{\vw}$ could be suboptimal, and we propose two algorithms to overcome this drawback. We first propose iSILO, an iterative method detailed in Algorithm \ref{alg:isilo}. Given the model in \eqref{eq:sim}, \isilo minimizes the squared loss with a sparsity penalty to estimate $\hat{\vw}, \hat{g}$:
\begin{align}
\label{eqn:noncal}
\hat{\vw}, \hat{g} =   \arg \min_{\vw,g} \frac{1}{n} \sum_{i=1}^n  (y_i-g(\vw^\top \vx_i))^2+\lambda \|\vw\|_1.
\end{align} 
We adopt an alternating minimization prodecure. In iteration $t$, given $g_{t-1}$, we would ideally perform a proximal point update w.r.t. $\vw$ to obtain
\begin{equation*}
\vw_{t}  = \mbox{Prox}_{\lambda\eta, \|\cdot \|_1} \left(\vw_{t-1}-\frac{\eta}{n}\sum_{i=1}^n(g_{t-1}(\vw_{t-1}^\top
\vx_i)-y_i) g_{t-1}'(\vw^T_{t-1} \vx_i) \vx_i\right)
\end{equation*}
where $\mbox{Prox}(\cdot)$ is the soft thresholding operator associated with the $\|\cdot\|_1$ norm, $\eta>0$ is an appropriate step size, and $g_t'$ is the derivative of $g_t$. Unfortunately, the above gradient step requires us to estimate the derivative of $g_t$, which can be difficult. So, instead of performing the above proximal gradient update, we instead perform  a proximal perceptron type update similar in spirit to \cite{sim_ravi, sim_sham}, by replacing $g_{t-1}'$ by the Lipschitz constant of $g_{t-1}$. Since $g_{t-1}$ is obtained using the LPAV algorithm, $g_{t-1}$ is $1-$ Lipschitz. Note that unlike the perceptron, we have a non unity step size. This leads to the following update equation
\begin{equation}
\label{eqn:noncal_grad}
\vw_{t}  = \mbox{Prox}_{\lambda\eta, \|\cdot \|_1} \left(\vw_{t-1}-\frac{\eta}{n}\sum_{i=1}^n(g_{t-1}(\vw_{t-1}^\top
\vx_i)-y_i) \vx_i\right)
\end{equation}
Given $\vw_t$ in iteration $t$, \isilo updates $g_t$ to be the solution to the LPAV problem with $p_i=\vw_t^\top\vx_i$. 

The non-convexity of \eqref{eqn:noncal} requires us to to perform a book-keeping procedure that keeps track of the best estimate of $\hat{g},\vwhat$ by calculating the MSE of the current hypothesis on a held-out validation set. 
This is done in steps 5-9 and  12-16 of Algorithms \ref{alg:isilo}. Similar book-keeping procedures have been used in the Isotron, and Slisotron algorithms of~\cite{sim_ravi,sim_sham}.
 \begin{algorithm}[t]
  \caption{\isilo}
   \label{alg:isilo}
  \begin{algorithmic}[1]
  \REQUIRE  Data: $X=[\vx_1,\ldots,\vx_n]$, Labels:  $\vy=[y_1,\ldots,y_n]^\top$,   Regularization: $\lambda$, Step size $\eta$, Initial parameters: $g_0$ is 1-Lipschitz, monotonic function, $\vw_0\in \bbR^d$, Iterations: $T>0$.
  \STATE Initialize $\vwhat=\vw_0, \hat{g}=g_0$.
  \STATE $\opterr=MSE(\vw_0,g_{0})$
  \FOR{t=1,\ldots T}
  \STATE Perform the update shown in Equation~ \eqref{eqn:noncal_grad}  to get $\vw_t$.
 \STATE Calculate $\err= MSE(\vw_t,g_{t-1})$. 
  \IF{$\err \leq \opterr$}
  \STATE $\opterr= \err.$
  \STATE $\hat{\vw}=\vw_t,\hat{g}=g_{t-1}$ 
  \ENDIF
  \STATE Obtain $g_t$ by solving problem~\eqref{opt:lpav} with $p_i=\vw_t^\top\vx_i$ and linear interpolation \eqref{eqn:interpolate}
  \STATE Calculate ${err}=MSE(\vw_t,g_{t})$.
  \IF{$\err\leq \opterr$}
  \STATE ${\opterr}={\err}.$
  \STATE $\hat{\vw}=\vw_t,\hat{g}=g_{t}$ 
  \ENDIF
  \ENDFOR 
  \STATE Output $\hat{\vw}, ~\ \hat{g}$
  \end{algorithmic}
\end{algorithm} 

\subsection{ciSILO: Iterative SILO with calibrated loss} 
\label{sec:cisilo}
\isilo like the Slisotron algorithm~\cite{sim_sham} use a squared loss function and an approximate gradient descent method to estimate $\wstar$. These methods do not take into account the derivative of the estimate of the transfer function while taking gradient descent steps. 
We now propose ciSILO, a version of \silo that uses a \emph{calibrated} loss function that adapts to the SIM that we are trying to learn. 

Suppose $\gstar$ was known. Let $\Phistar: \R \rightarrow \R$ be a function such that
$\Phistar' = \gstar$. Since $\gstar$ is monotonically increasing, $\Phistar$ is convex, and we
can learn $\hat{w}$ by solving the following convex program:
\begin{equation}
\label{def_glm_log_likelihood_single}
\hat{\vw} :=  \frac{1}{n} \sum_{i = 1}^n \Phi_\star(\vw^\top\vx_i) - \vy_i \vw^\top \vx_i+\lambda \|\vw\|_1 
\end{equation}
When the transfer function is linear, $\Phistar$ is a quadratic function, and we obtain the standard Lasso problem that minimizes squared loss with $\ell_1$ penalty. When the transfer function is the logit function, \eqref{def_glm_log_likelihood_single} reduces to sparse logistic regression. Modulo, the $\ell_1$ penalty term the above objective is a sample version of the following stochastic optimization problem:
\begin{equation}
\min_{\vw}\bbE [\Phi_\star(\vw^\top \vx)-y \vw^\top \vx ].
\end{equation}                         
If $\Phistar'=\gstar$, then the optimal solution to the above problem corresponds to the single index model that satisfies $\bbE[Y|X=x]=\gstar(\vwstar^\top\vx)$. Hence the above \textit{calibrated loss function} takes into account the transfer function $\gstar$ used in the SIM via $\Phi_\star$ and automatically adapts to the SIM from which the data is generated. When $\gstar$ is unknown, we instead consider the following optimization problem:
\begin{align}
\label{opt:sim_calibrated_loss}
\hat{\vw},\hat{g} = \arg \min_{\vw,g} \frac{1}{n} \sum_{i = 1}^n \Phi(\vw^\top\vx_i) - \vy_i \vw^\top\vx_i + \lambda \| \vw \|_1\quad
\text{~s.t.~} \quad g=\Phi' \in \cG
\end{align}
where the set $\cG=\{g:\bbR\rightarrow \bbR~\text{is a 1-Lipschitz, monotonic function}\}$. Note that the above optimization problem optimizes for $g$ via its integral $\Phi$. \cisilo solves the above optimization problem by iteratively minimizing for $w,g$. The pseudo-code for \cisilo is given in Algorithm~\ref{alg:cisilo}. 
There are three key update procedures performed in each iteration of ciSILO, which we explain below:

In Step 4, ciSILO fixes  $g$ to $g_{t-1}$ and performs one step of a proximal point update on the objective in problem~\eqref{opt:sim_calibrated_loss} w.r.t. $\vw$
to get:
\begin{equation}
\label{eqn:prox_step_csilo}
\vw_{t}  = \mbox{Prox}_{\lambda\eta, \|\cdot \|_1} \left(\vw_{t-1}-\frac{\eta}{n}\sum_{i=1}^n(g_{t-1}(\vw_{t-1}^\top
\vx_i)-y_i)\vx_i\right).
\end{equation}
This step is identical to the update step in \isilo except that the $g'_{t-1}$ does not feature in this update. Thus, the proximal point steps using a calibrated loss function can be performed \emph{exactly} unlike the proximal point steps in \isilo.

The use of a calibrated loss function brings with it another challenge: The LPAV procedure, which was designed to minimize the squared loss, can no longer be used in \cisilo to estimate $\gstar$. \cisilo instead uses a novel quadratic program to efficiently estimate $\gstar$. From the first order optimality conditions of the optimization problem~\eqref{opt:sim_calibrated_loss} for $\vw$ at $\vw_t$ we get that the optimal function $g_t$ should satisfy
\begin{equation}
\label{eqn:subgrad}
\frac{1}{n}\sum_{i=1}^n (g_t(\vw_{t}^\top x_i)-y_i)x_i+\lambda \vbeta_{t} = 0, \qquad \vbeta_{t}\in \partial ||\vw_{t}||_1.
\end{equation}
$g_t$ is updated such that L.H.S. of \eqref{eqn:subgrad} has the smallest possible norm. This can be cast as a quadratic program (QP) as follows: Define, $\vp=[p_1,\ldots,p_n]^\top$, where $p_i=\vw_t^\top x_i$ and $\vz=[z_1,\ldots,z_n]^\top$, where $z_i=g_t(p_i)$. Let $\mX=[\vx_1,\ldots,\vx_n]$x be a $d\times n$ data matrix. Let $\vq =n\lambda \beta - \mX^\top y$. Now, solve the problem
\begin{equation}
	\label{opt:qpfit}
  \centering
  \begin{aligned}
  \min_{\vz} &~ \lVert \mX^\top \vz+\vq\rVert_2^2 \\
  \textbf{s.t.}& ~\ 0 \leq \vz_i \leq 1 ~\ \forall i ~\   \textbf{and }  ~\ 0 \leq z_j-z_i \leq p_j-p_i~\textbf{if~} p_i\leq p_j
  \end{aligned}
  \end{equation}
We call optimization problem \eqref{opt:qpfit} \textbf{QPFit}, which is different from the LPAV given that it is derived from optimizing a calibrated loss function, which could be very different from the squared loss.
 \begin{algorithm}[!h]
  \caption{\cisilo}
   \label{alg:cisilo}
  \begin{algorithmic}[1]
  \REQUIRE  Data: $\mX=[\vx_1,\ldots,\vx_n]$, Labels  $\vy=[y_1,\ldots,y_n]^\top$, , Regularization parameter $\lambda$, step size $\eta$, Initial parameters: $\vw_0\in \bbR^d,g_0:\bbR\rightarrow \bbR$ is 1-Lipschitz, monotonic function.
  \STATE Initialize $\vwhat=\vw_0, \hat{g}=g_0$.
  \STATE $\opterr=MSE(\vw_0,g_{0})$
  \FOR{t=1,2,\ldots T}
  \STATE Perform the update step shown in Equation~\eqref{eqn:prox_step_csilo} to obtain $\vw_t$.
 \STATE Calculate $\err= MSE(\vw_t,g_{t-1})$. 
  \IF{$\err \leq \opterr$}
  \STATE $\opterr= \err.$
  \STATE $\hat{\vw}=\vw_t,\hat{g}=g_{t-1}$ 
  \ENDIF
  \STATE Calculate: $\vp \leftarrow \mX \vw_t, \beta\leftarrow \partial \|\vw_t\|_1, \vq \leftarrow n\lambda \beta - \mX^\top y$
  \STATE Obtain $g_t$ by solving problem~\eqref{opt:qpfit} and linear interpolation.
  \STATE Calculate ${err}=MSE(\vw_t,g_{t})$.
  \IF{$\err\leq \opterr$}
  \STATE ${\opterr}={\err}.$
  \STATE $\hat{\vw}=\vw_t,\hat{g}=g_{t}$ 
  \ENDIF
  \ENDFOR 
  \STATE Output $\hat{\vw}, ~\ \hat{g}$
  \end{algorithmic}
\end{algorithm}

\subsection{Initializing \isilo and \cisilo}
Since both \isilo and \cisilo are non-convex, alternating minimization procedures, a good initialization is key to achieving good performance.  A simple initialization would be to choose $\vw^0$ randomly and $g^0$ to be the identity function. However, we initialize both methods with $\hat{\vw}, \hat{g}$ obtained by running the (efficient) \silo algorithm from Section \ref{sec:silo}. We demonstrate in the next section that this yields very good theoretical guarantees, as well as good empirical performance in Section \ref{sec:exp}.
\paragraph{Remarks :}
Like in \isilo we perform book-keeping steps in \cisilo too.   
Since obtaining exact or approximate gradients in \isilo and \cisilo are easy  we use first order methods to solve for $\hat{\vw}$. Using line search methods in ciSILO, to compute step sizes, would require evaluating the calibrated loss function. This can be computationally intensive, since we have access to the calibrated loss function only via its gradient. Hence, in iSILO, and \cisilo we use a fixed step size to perform our updates. Despite the use of fixed step size, we show empirically that \isilo is often as competitive and sometimes better at making predictions than GLM based methods with optimal step sizes, and \cisilo is significantly superior. 

\section{Theoretical analysis of SILO, iSILO and \cisilo}
\label{sec:theory}
In this section, we analyze the excess risk of the predictors output by iSILO, and \cisilo. For a given hypothesis $\hat{h}(x)=\hat{g}(\hat{\vw}^\top\vx)$, define $\err(h) := \mathbb{E} \left( h(\vx) - \vy \right)^2 $. The excess risk is then defined as
\begin{equation}
\label{defer}
\mathcal{E}(\hat{h}):= \err(\hat{h}) - \err(h_\star)=\bbE(\vy-\hat{h}(\vx))^2-\bbE(\vy-\gstar(\vw_\star^\top \vx))^2  
\end{equation}
We first list the technical assumptions we make:
%
%
\begin{enumerate}
\item[A1.] The data $\vx_1,\ldots,\vx_n$ is sampled i.i.d. from the standard multivariate Gaussian distribution.
\item[A2.] $\bbE[Y|X=x]=\gstar(\vwstar^\top \vx)$, and $0\leq Y\leq 1$,
\item[A3.] $\gstar$ is monotonic and $1-$ Lipschitz,
\item[A4.] $\|\vwstar\|_0\leq s,\|\vwstar\|_2\leq 1,\|\hat{\vw}\|_0\leq k$, and $k  \ll d$.
\end{enumerate}


We provide sketches of relevant results in this section, and refer the interested reader to the Appendix for detailed proofs. 
Our first main result is an excess risk bounds for SILO:
\begin{thm}
\label{thm:hsim1}
Let $\hat{h}(x)=\hat{g}(\hat{\vw}^\top\vx)$ be the hypothesis output by SILO. Let $\theta=\bbE_{\mu\sim N(0,1)} \gstar(\mu)\mu>0$. Then under assumptions A1-A4, the excess risk  of the predictor $\hat{h}$ is, with probability at least $1-\delta$,  bounded from above by
\begin{equation}
\mathcal{E}(\hat{h})= \tilde{O}\left(\frac{(s+k)\log(2d)}{\theta}\sqrt{\frac{s}{n}}+\frac{1}{\sqrt{\theta}}\left(\frac{s}{n}\right)^\frac{1}{4}\sqrt{(s+k)\log(2d)}\right)
\end{equation}
where $\tilde{O}$ hides factors that are poly-logarithmic in $n,d,\frac{1}{\delta},s \text{ and } k$.
\end{thm}
\paragraph{Proof Sketch:} For notational convenience, denote by $\epsilon^2 =\frac{1}{\theta}\sqrt{\frac{Cs\log(2d/s)}{n}}$, where $C>0$ is a universal constant. WLOG, we can assume that $\|\hat{\vw}\|_0\leq s$. Our assumption on the sparsity of $\hat{\vw}$ is pretty lenient, and is most often satisfied in practice. Also, since $\hat{\vw}$ is obtained from \silo, we have $\|\hat{\vw}\|_2\leq 1,\|\hat{\vw}\|_1\leq \sqrt{s}$. From a result of Plan and Vershynin~\cite[Corollary 3.1]{plan_1bit} (Lemma 4 in appendix), we know that $\|\wstar-\what\|_2^2\leq \epsilon^2$.
The excess risk $\cE(\hath)$ can be bounded as follows.
\begin{align*}
\notag
\cE(\hath)&=\bbE [(\hatu(\what^\top \vx)-y)^2- (\gstar(\wstar^\top \vx)-y)^2]
=\bbE (\hatu(\what^\top \vx)-\gstar(\wstar^\top \vx))^2\\
\notag
&=\bbE (\hatu(\what^\top \vx)-\ghat(\wstar^\top\vx)+\ghat(\wstar^\top\vx)-\gstar(\wstar^\top \vx))^2\\
\notag
&\leq 2(s+k)\epsilon^2\log(2d)+2~\bbE(\ghat(\wstar^\top\vx)-\gstar(\wstar^\top \vx))^2~\text{with probability at least }~1-\delta
\end{align*}
where we used the fact that $\ghat$ is 1-Lipschitz, and upper bounds on the expected suprema of a collection of Gaussian random variables. Next, we shall bound the R.H.S. of the above equation. 
\begin{align*}
\bbE(\ghat(\wstar^\top\vx)-\gstar(\wstar^\top \vx))^2&\leqa  \bbE(\ghat(\wstar^\top\vx)-y)^2-\bbE (\gstar(\wstar^\top \vx)-y)^2\\
&\leqb \frac{1}{n}\sum_{i=1}^n(\ghat(\wstar^\top\vx_i)-y_i)^2- (\gstar(\wstar^\top \vx_i)-y_i)^2+\tilde{O}\left(\frac{(s\log(2d))^{1/4}}{\sqrt{n}}\right)\label{eqn:this2}
\end{align*}
In inequality (a) we used a certain projection inequality for convex sets (see Lemma 1 in appendix). To obtain inequality (b) we replace the expected value quantities with their empirical versions, plus deviation terms. Via standard application of large deviation inequalities, it is possible to establish that these deviations are $\tilde{O}(\frac{(s\log(2d))^{1/4}}{\sqrt{n}})$ (see Lemma 5 in appendix). The proof concludes by upper bounding the empirical term in the above equation using optimality  of $\hat{g}$ and properties of maxima of a collection of Gaussian random variables.

Our next result is an upper bound on the excess risk bounds of \isilo and ciSILO:
\begin{thm} 
\label{thm:hsimm}
Suppose  $\hat{g},\hat{\vw}$ are the outputs of  \silo on our data. Let $\hat{h}(\vx)=\hat{g}(\hat{\vw}^\top\vx)$ be the hypothesis corresponding to these outputs. Let $h_{\star}(x)\defeq\gstar(\wstar^\top\vx)$. Now, let $\hat{h}_T$ be the output of \cisilo obtained by using $\hat{g},\hat{\vw}$ as initializers. Then under the assumptions A1-A4, with high probability we can bound the excess risk of $\hat{h}_T$ by
\begin{equation*}
\mathcal{E}(\hat{h}_T)\leq \tilde{O}\left(\frac{(s+k)\log(2d)}{\theta}\sqrt{\frac{s}{n}}+\frac{1}{\sqrt{\theta}}\left(\frac{s}{n}\right)^\frac{1}{4}\sqrt{(s+k)\log(2d)}\right)+\sqrt{\frac{s\log^2(2d+1)}{n}}+\sqrt{\frac{1}{n}}
\end{equation*}
where $\tilde{O}$ hides factors that are poly-logarithmic in $n,d,\frac{1}{\delta},s,k$. Moreover, the same excess risk guarantees hold for $\hat{h}_T$ obtained by running \isilo.
\end{thm}
\paragraph{Proof Sketch :} From Theorem~\ref{thm:hsim1} we know that 
\begin{equation*}
\mathcal{E}(\hat{h})=\err(\hat{h})-\err(h_{*})\leq \tilde{O}\left(\frac{(s+k)\log(2d)}{\theta}\sqrt{\frac{s}{n}}+\frac{1}{\sqrt{\theta}}\left(\frac{s}{n}\right)^\frac{1}{4}\sqrt{(s+k)\log(2d)}\right)
\end{equation*}
Using standard large deviation arguments (see Lemma 6 in appendix) we can claim that $|\err(\hat{h})-\widehat{\err}(\hat{h})|=\tilde{O}(\sqrt{\frac{s}{n}})$ with probability at least $1-\delta$. This gives us
\begin{align*}
\widehat{\err}(\hat{h})&=\err(\hat{h})+\tilde{O}\left(\sqrt{\frac{s}{n}}\right)=\err(h_{\star})+\err(\hat{h})-\err(h_{\star})+\tilde{O}\left(\sqrt{\frac{s}{n}}\right)\nonumber\\
	&=\err(h_{\star})+\tilde{O}\left(\frac{(s+k)\log(2d)}{\theta}\sqrt{\frac{s}{n}}+\frac{1}{\sqrt{\theta}}\left(\frac{s}{n}\right)^\frac{1}{4}\sqrt{(s+k)\log(2d)}\right)+\sqrt{\frac{s\log^2(2d+1)}{n}}.
\end{align*}
Now consider $\hat{h}_{T}$ obtained by running either \cisilo or \isilo for $T$ iterations, when initialized with $\hat{w},\hat{g}$ obtained by running \silo first on the data. Since $\hat{h}_T$ is chosen by using a held-out validation set as the iterate corresponding to the smallest validation error, we can claim via Hoeffding inequality that the empirical error of $\hat{h}_T$ cannot be too much larger than that of $\hat{h}$ (for otherwise $\hat{h}_T$ will not be the iterate with the smallest validation error). Precisely, if the validation set is of size $n$, then with high probability
$\widehat{\err}(\hat{h}_T)\leq \widehat{\err}(\hat{h})+\tilde{O}\left(\frac{1}{\sqrt{n}}\right)$.
Using the above inequalities, and via standard large deviation arguments to bound $|\err(\hat{h}_T)-\widehat{\err}(\hat{h}_T)|$ we get the desired result.

\paragraph{Remarks :} In the bound of Theorem \ref{thm:hsimm}, the first term in $\tilde{O}$ dominates, and the excess risk bound is essentially $\tilde{O}\left( \frac{(s+k)\log(2d)}{\theta}\sqrt{\frac{s}{n}} \right)$.  Also, using the output of \silo to initialize \isilo and \cisilo yields  strong theoretical guarantees. 

\paragraph{The constant $\theta$ in our results:} $\theta$ acts like the signal to noise ratio in our results. The larger $\theta$ is, the better our bound gets. For example, for the logistic model, $\theta$ is approximately the norm of the data $(\sim \sqrt{\log(d)} )$. For measurements of the form $\vy = sign(\vx^T\vw), ~\ \theta$ is a constant. $\theta < 0$ can be easily tackled by reversing the signs of $\vy$, and $\theta = 0$ implies that the data and observations are uncorrelated, and naturally any error bound  will be meaningless. 

\paragraph{Comparisons to existing results in low dimensions: } In~\cite{sim_sham} the authors obtained dimension dependent as well as dimension independent bounds on the prediction error for the Slisotron algorithm for the SIM problem. However, these results were obtained under the  restrictive assumption that $\|\vwstar\|_2\leq W,\|\vx\|_2\leq B$, and both $W,B$ are fixed and independent of dimensions.~\footnote{In their analysis $B=1$.} In order to carry through a correct high-dimensional analysis, one needs to let either $W$ or $B$ or both grow with $d$. In our analysis, we assume that the data is sampled from a standard multi-variate Gaussian, and hence $\|\vx\|_2\leq \sqrt{d}$ with high probability. If one were to replace $B$ with $\sqrt{d}$ in the results of~\cite{sim_sham}, then the excess risk of their predictor would scale as  $\min\{\frac{d}{n^{1/3}},\frac{\sqrt{d}}{n^{1/4}}\}$, and since $d \gg n$, their bounds are meaningless in the high-dimensional setting. In contrast our results in Theorem~\ref{thm:hsimm} have a (poly)-logarithmic dependence on $d$, and hence are useful in the high dimensional setting studied in this paper. The same arguments apply to the results of~\cite{sim_ravi}, where in addition one needs a fresh batch of samples at each run. 

\section{Experimental results}
\label{sec:exp}
\begin{figure}[t]
\centering
\includegraphics[width = 130mm, height = 40mm]{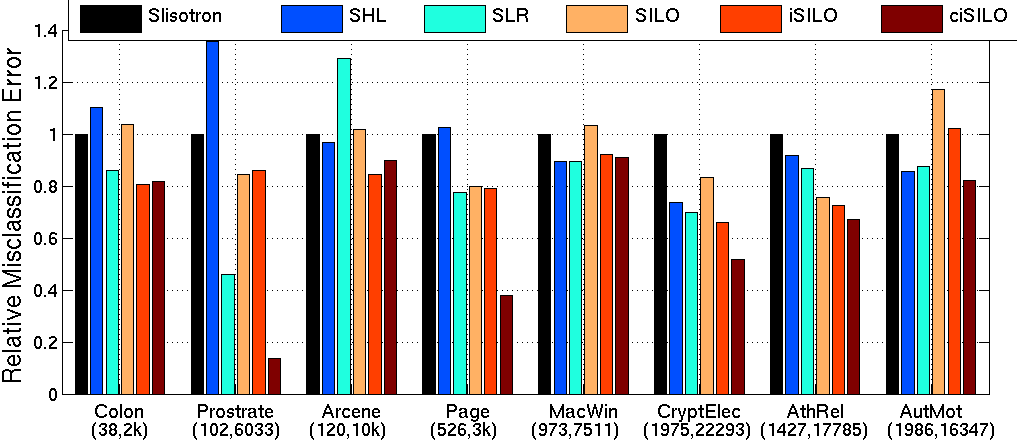}
\caption{Errors rates are normalized so that the Slisotron has an error of 1. Note that \cisilo consistently outperforms all other methods, and \isilo is very competitive. The numbers below each dataset refer to $(n,d)$}
\label{fig:results}
\end{figure}
We tested our algorithms SILO, iSILO, and \cisilo on many real world high dimensional datasets. For comparison with methods that assume $g$ known, we used Sparse Logistic Regression (SLR), and Sparse Squared Hinge Loss minimization (SHL) \cite{l1gen} \footnote{code downloaded from \url{http://www.cs.ubc.ca/~schmidtm/Software/L1General.html}} .  We also tested the Slisotron \cite{sim_sham} algorithm designed for low-dimensional SIM. For each dataset we randomly chose $60 \%$ of the data for training, and $20
\%$ each for validation and testing. The parameters $\lambda,\eta$ are chosen via validation. Mac-Win, Crypt-Elec, Atheism-Religion and Auto-Motorcycle are from the 20 Newsgroups dataset. Arcene is from the NIPS challenge \footnote{\url{http://www.nipsfsc.ecs.soton.ac.uk/datasets/}}, and the Page dataset is obtained form the WebKB dataset~\cite{nigam2001using}~\footnote{\url{http://vikas.sindhwani.org/manifoldregularization.html}}.  Prostrate and Colon cancer datasets are available online \footnote{\url{http://www.stat.cmu.edu/~jiashun/Research/software/HCClassification/Prostate/}} .

Figure \ref{fig:results} shows the misclassification error obtained on the test set.  We show results for 8 datasets of varying size. Additional results are available in the supplementary material. Since the datasets (and errors) are varied, we normalize the error rates so that the Slisotron has unit error. As we can see from these results, using the calibrated loss in \cisilo yields the best performance in all the datasets considered, except MacWin. \isilo is as good as or better than SLR in 6/8 cases. It is encouraging to note that \isilo and \cisilo do well despite not having the luxury of choosing optimal step sizes at each iteration. Finally, the relatively poor performance of \silo underlines the importance of iterative methods in the SIM learning setting. 

\section{Conclusions}
\label{sec:conc}
\vspace{-2mm}
In this paper, we introduced a suite of algorithms based on sparse parameter estimation for learning single index models in the high dimensional setting.  We derived excess risk guarantees for the proposed methods. Our algorithm employing a calibrated loss and a novel quadratic programming method to fit the transfer function achieves superior results compared to standard high dimensional classification methods based on minimizing the logistic or the hinge loss. In the future we plan to investigate learning single index models with structural constraints other than sparsity such as low rank, group sparsity, and indeed  other very general constraints.
\bibliography{HSIM_arxiv}
\bibliographystyle{plain}
\newpage
\appendix
\section{Preliminaries}
We shall need a few definitions and a few important lemmas and propositions before we can state the proofs of our theorems.
We shall consider the following function class.
\begin{equation}
\mathcal{G}=\{g:[-W,W]\rightarrow [0,1], g~\text{ is 1-Lipschitz and monotonic}\}.
\end{equation}
Though the above definition of $\cG$ uses an unspecified parameter $W$, most often we shall use $W=\sqrt{s\log(2d)}$.
The following result concerning suprema of a collection of i.i.d. Gaussian random variables is standard and we shall state it without proof.
\begin{proposition}
\label{prop:gaussian}
Let $[g_i]_{i = 1}^m$ be a collection of  $m$ i.i.d. Gaussian random variables with mean $0$ and variance $\sigma^2$. Then, 
\[
\max_{i \in [m]} | g_i | \leq \sigma\left(\sqrt{\log(2m)}+\sqrt{2\log(2/\delta)}\right) ~\ \textbf{w.p.} ~\ \geq 1 - \delta
\]
\end{proposition}
The next lemma is standard and a proof can be found in Lemma 9 in~\cite{sim_sham}.
\begin{lemma}
\label{lem:project}
Let $\cF$ be a convex class of functions, and let $f^*=\arg\min_{f\in \cF}\bbE (f(x)-y)^2$. Suppose that $\bbE[Y|X=x]=\gstar(\wstar^\top\vx)$ for some $\gstar\in\cG$. Then for any $f\in\cF$, the following holds true
\begin{equation}
\bbE[(f(x)-y)^2]-\bbE[(f(x)-y)^2]\geq\bbE[(f(x)-f^{*}(x))^2]
\end{equation}
\end{lemma}

\begin{lemma} 
\label{lem:sip}
Let $\vx\in\bbR^d$ be a standard normal random vector. Then with probability at least $1-\delta$
\begin{equation*}
\wstar^\top\vx\leq \tilde{O}(\sqrt{s\log(2d)})
\end{equation*}
\end{lemma}
\begin{proof}
The proof follows immediately from Proposition~\eqref{prop:gaussian} and the fact that $\|\wstar\|_1\leq \sqrt{s}$.
\end{proof}

\begin{lemma}
\label{lem:ip}
Let $\ve\in\bbR^d$ be such that $\|\ve\|_0\leq s+k$ and $\|\ve\|_2\leq \epsilon$. Let $\vx$ be a standard normal random vector. Then with probability at least $1-\delta$
\begin{equation*}
\ve^\top\vx\leq \tilde{O}(\epsilon\sqrt{(s+k)\log(2d)})
\end{equation*}
\end{lemma}
\begin{proof}
Let $\ve=[e_1,\ldots,e_d]$. Similarly, let $\vx=[x_1,\ldots,x_d]$. We then have
\begin{align}
\ve^\top\vx&=\sum_{i=1}^d e_ix_i\\
&\leq \max |x_i|\sum_{i=1}^{d}|e_i|\\
&\leqa \sqrt{\log(2d)}\sum_{i=1}^{s+k}|e_i|~ w.p~ 1-\delta\\
& \leqb \epsilon\sqrt{(s+k)\log(2d)}.
\end{align}
In obtaining inequality (a) we used the fact that the max of the absolute value of $d$ Gaussian random variables is bounded by $\sqrt{\log(2d)}$. In equality (b) we used the fact that $\|\ve\|_0\leq s+k$, and hence only $s+k$ of the elements of $\ve$ are non-zero.
\end{proof}
We next need the following important result (Corollary 3.1 in~\cite{plan_1bit})
\begin{lemma}. 
\label{lem:pv_bound}
Let $\cW=\{\vw\in \bbR^d:\|\vw\|_2\leq 1, \|\vw\|_1\leq \sqrt{s}\}$. Let $\what$ be obtained from \silo, shown in the main paper. Suppose, $\hat{\vw}\in \cW$. Let $\vx_1,\ldots\vx_n$ be $n$ independent Gaussian random vectors. Assume that the measurements $\bbE[Y|X=x]=\gstar(\wstar^\top\vx)$, where $\|\wstar\|_2\leq 1,\|\wstar\|_{0}\leq s$. Then with probability at least $1-\delta$, the solution $\what$ obtained from \silo satisfies the inequality
\begin{equation*}
\|\what-\wstar\|_2^2\leq \epsilon^2 \leq \frac{1}{\theta}\sqrt{\frac{Cs\log(2d/s)}{n}},
\end{equation*}
where $C>0$ is a universal constant, and $\theta=\bbE_{\mu\sim N(0,1)} \gstar(\mu)\mu$
\end{lemma}
\begin{lemma}
\label{lem:large_dev}
With probability at least $1-\delta$
\begin{equation}
\bbE(\ghat(\wstar^\top\vx)-y)^2-\bbE (\gstar(\wstar^\top \vx)-y)^2
\leq  \frac{1}{n}\sum_{i=1}^n(\ghat(\wstar^\top\vx_i)-y_i)^2- (\gstar(\wstar^\top \vx_i)-y_i)^2+\tilde{O}\left(\sqrt{\frac{W}{n}}\right)
\end{equation}
where $\tilde{O}$ hides factors that are (poly)-logarithmic in $n,\frac{1}{\delta}$
\end{lemma}
\begin{proof}
 From Lemma 6 (i) in ~\cite{sim_sham} we know that 
\begin{equation}
\cN_2(r,\cG,\vz_{1},\ldots,\vz_{n})\leq \cN_{\infty}(r,\cG)\leq \frac{1}{r}2^{\frac{2W}{r}},
\end{equation}
where $\cN_2(r,\cG,\vz_{1},\ldots,\vz_{n})$ is the $L_2$ empirical covering number of function class $\cG$ at radius $r$, and $\cN_{\infty}(r,\cG)$ is the $L_{\infty}$ covering number. Using Dudley entropy integral, we can upper bound the empirical Rademacher complexity by 
\begin{equation}
\hat{R}_n(\cG)=\inf_{\alpha> 0} 4\alpha+10\int_{\alpha}^1 \sqrt{\frac{\log(1/r)+\frac{2W}{r}}{n}~\mathrm{d}r}\leq \frac{40\sqrt{W}}{\sqrt{n}}.
\end{equation}
Hence, via standard large deviation inequalities we can claim that 
\begin{equation}
\bbE[(\ghat(\wstar^\top\vx)-y)^2]\leq \frac{1}{n}\sum(\ghat(\wstar^\top\vx)-y)^2+O(\sqrt{\frac{W}{n}}).
\end{equation}
Similarly via standard concentration inequalities we can claim that with probability at least $1-\delta$,
\begin{equation}
|\bbE[(\gstar(\wstar^\top\vx)-y)^2]-\frac{1}{n}\sum_i(\gstar(\wstar^\top\vx)-y)^2|\leq O(\sqrt{\frac{\log(2/\delta)}{n}})
\end{equation}
and hence putting together the above two inequalities the desired result follows.
\end{proof}
\section{Proof of Theorem~\ref{thm:hsim1}}
For notational convenience, denote by $\epsilon^2 =\frac{1}{\theta}\sqrt{\frac{Cs\log(2d/s)}{n}}$, where $C>0$ is a universal constant. Since, $\hat{\vw}$ is obtained from SILO, we have $\|\hat{\vw}\|_2\leq 1,\|\hat{\vw}\|_1\leq \sqrt{s}$.
The excess risk $\cE(\hath)$ can be bounded as follows.
\begin{align}
\cE(\hath)&=\bbE [(\hatu(\what^\top \vx)-y)^2- (\gstar(\wstar^\top \vx)-y)^2]\nonumber\\
&=\bbE (\hatu(\what^\top \vx)-\gstar(\wstar^\top \vx))^2\nonumber\\
&=\bbE (\hatu(\what^\top \vx)-\ghat(\wstar^\top\vx)+\ghat(\wstar^\top\vx)-\gstar(\wstar^\top \vx))^2\nonumber\\
&\leq 2~\bbE (\hatu(\what^\top \vx)-\ghat(\wstar^\top\vx))^2+2~\bbE(\ghat(\wstar^\top\vx)-\gstar(\wstar^\top \vx))^2\nonumber\\
&\leqa 2~\bbE ((\what-\wstar)^\top\vx)^2+2~\bbE(\ghat(\wstar^\top\vx)-\gstar(\wstar^\top \vx))^2\nonumber\\
&\leqb 4s\epsilon^2\log(2d)+2~\bbE(\ghat(\wstar^\top\vx)-\gstar(\wstar^\top \vx))^2~\text{with probability at least }~1-\delta\label{eqn:this1}
\end{align}
Where in order to obtain inequality (a) we used the fact that $\ghat$ is 1-Lipschitz, and in order to obtain inequality (b) we used Lemma~\eqref{lem:ip}. We shall now bound the R.H.S. of inequality~\ref{eqn:this1}. We do this as follows
\begin{align}
\bbE(\ghat(\wstar^\top\vx)-\gstar(\wstar^\top \vx))^2&\leqa  \bbE(\ghat(\wstar^\top\vx)-y)^2-\bbE (\gstar(\wstar^\top \vx)-y)^2\\
&\leqb \frac{1}{n}\sum_{i=1}^n(\ghat(\wstar^\top\vx_i)-y_i)^2- (\gstar(\wstar^\top \vx_i)-y_i)^2+\Delta_1\label{eqn:this2}
\end{align}
In inequality (a) we used Lemma~\ref{lem:project} with the function class $\cF=\cG\circ \wstar$. In inequality (b) we used Lemma~\eqref{lem:large_dev} the expectation quantity in terms of its empirical quantity, with $\Delta_1$ set to the maximum value of $\wstar^\top \vx_i$. We know, from Lemma~\ref{lem:sip} that this max value is $\sqrt{s\log(2d)}$ with probability at least $1-\delta$. Hence by substituting $W=\sqrt{s\log(2d)}$ for $W$, we get $\Delta_1=O\left(\sqrt{\frac{\sqrt{s\log(2d)}}{n}}\right)$. Next we shall try to upper bound the empirical term in the above equation.

We have
\begin{align}
\frac{1}{n}\sum_{i=1}^n(\ghat(\wstar^\top\vx_i)-y_i)^2- (\gstar(\wstar^\top \vx_i)-y_i)^2&=
\frac{1}{n}\sum_{i=1}^n(\ghat(\what^\top\vx_i)-y_i-\ghat(\what^\top\vx_i)+\ghat(\wstar^\top\vx_i))^2- \nonumber\\
&\hspace{10pt}\frac{1}{n}\sum_{i=1}^n (\gstar(\what^\top\vx_i)-y_i-\gstar(\what^\top\vx_i)+\gstar(\wstar^\top\vx_i))^2 \nonumber\\
&=\underbrace{\frac{1}{n}\sum_{i=1}^n (\ghat(\what^\top\vx_i)-y_i)^2-\frac{1}{n}\sum_{i=1}^n (\gstar(\what^\top\vx_i)-y_i)^2}_{\leq 0} \nonumber\\
&\hspace{10pt}+\underbrace{\frac{1}{n}\sum_{i=1}^n (\ghat(\what^\top\vx_i)-\ghat(\wstar^\top\vx_i))^2}_{T_1}-\underbrace{\frac{1}{n}\sum_{i=1}^n(\gstar(\wstar^\top\vx_i)-\ghat(\what^\top\vx_i))^2}_{\geq 0} \nonumber\\
&\hspace{10pt}+\underbrace{\frac{2}{n}\sum_{i=1}^n (\ghat(\what^\top\vx_i)-y_i)(\ghat(\what^\top\vx_i)-\ghat(\wstar^\top\vx_i))}_{T_2} \nonumber\\
&\hspace{10pt}-\underbrace{\frac{2}{n}\sum_{i=1}^n (\gstar(\what^\top\vx_i)-y_i)(\gstar(\what^\top\vx_i)-\gstar(\what^\top\vx_i))}_{T_3}
\end{align}
where the term marked as $\leq 0$ is negative because $\ghat$ is the solution to a minimization problem that minimizes the empirical squared error under monotonicity and 1-Lipschitz constraints. Since $\gstar$ is also monotonic and 1-Lipschitz the squared error corresponding to the predictor $\ghat(\what^\top \vx)$ should be smaller than the squared error corresponding to $\gstar(\what^\top \vx)$. The term marked as $\geq 0$ is positive because it is an average of squared quantities. We shall now bound $T_1, T_2, T_3$ as follows
\begin{align}
T_1&=\frac{1}{n}\sum_{i=1}^n (\ghat(\what^\top\vx_i)-\ghat(\wstar^\top\vx_i))^2\\
&\leqa \frac{1}{n}\sum_{i=1}^n ((\what-\wstar)^\top\vx_i))^2\\
&\leqb (s+k)\epsilon^2\log(2d)
\end{align}
where, to obtain inequality (a) we used the fact that $\ghat$ is 1-Lipschitz, and to obtain inequality (b) we used Lemma 2.

To upper bound $T_2$ we proceed as follows
\begin{align}
T_2&=\frac{2}{n}\sum_{i=1}^n (\ghat(\what^\top\vx_i)-y_i)(\ghat(\what^\top\vx_i)-\ghat(\wstar^\top\vx_i))\\
&\leqa \frac{2}{n}\sum_{i=1}^n |\ghat(\what^\top\vx_i)-\ghat(\wstar^\top\vx_i))|\\
&\leqb \epsilon\sqrt{(s+k)\log(2d)}
\end{align}
To obtain inequality (a) we used the fact that $|y_i-\ghat(\what^\top\vx_i)|\leq 1$, and to obtain inequality (b) we used the fact that $\ghat$ is 1-Lipschitz and Lemma~\ref{lem:ip}. 
The same reasoning can be applied to upper bound $T_3$ to get $T_3\leq \epsilon\sqrt{k\log(2d)}$.

Finally using lemma~\eqref{lem:pv_bound}, we know that $\|\wstar-\what\|_2^2= \epsilon^2\leq \tilde{O}(\frac{1}{\theta}\sqrt{\frac{s}{n}})$. Gathering all the terms, we get with probability at least $1-\delta$, 
\begin{equation}
\cE(\hat{h})=\tilde{O}\left(\frac{(s+k)\log(2d)}{\theta}\sqrt{\frac{s}{n}}+\frac{1}{\sqrt{\theta}}\left(\frac{s}{n}\right)^\frac{1}{4}\sqrt{(s+k)\log(2d)}\right)
\end{equation}
 where, $\theta=\bbE_{\mu\sim N(0,1)} g(\mu)\mu$ is a constant that depends on $\gstar$.
\section{Large Deviation Guarantees for \isilo,~\cisilo}
 \begin{lemma}
\label{thm:large_dev}
For any hypothesis $h(x)=g(\vw^\top\vx)$, where $\cW=\{\vw\in \bbR^d: \|\vw\|_1\leq
\sqrt{s},\|\vw\|_{2}\leq 1\}$, $g\in\cG,\vw\in\cW$, we have  
\begin{align*}
\err(h_T)&\leq \widehat{\err}(h_T)+\tilde{O}\left(\sqrt{\widehat{\err}(h_T)}\sqrt{\frac{s}{n}}\right),
\end{align*}
where the $\tilde{O}$ hides factors (poly) logarithmic in $d,n,1/\delta$.
In particular the above result also applies to $h_T$ which is the hypothesis obtained by running \isilo ~or \cisilo~ for $T$ iterations, and to $\hat{h}$, the hypothesis obtained by running \silo. 
\end{lemma}
Before we give the proof of this theorem, we would like to point out that our assumption that $\hat{\vw}\in\cW$ is not at all restrictive. In practice the result provided by the iterates of a proximal gradient method used in \hsim-M for a sufficiently large $\lambda$ are sparse. 
\begin{proof}
Consider the function class
$\cH=\{h(\vx)=g(\vw^\top\vx): \vw\in \cW, g\in \mathcal{G}\}$.  By construction, we are guaranteed that $h_T,\hat{h}\in \cH$, w.h.p., with $W=\sqrt{s\log(2d)}$. In order to establish a large deviation bound on the risk of $h_T$ we shall first calculate the worst case
Rademacher complexity of $\cH$. To do this, we establish $L_2$ covering number of the function class $\mathcal{H}$ by establishing $L_{\infty}$ covering number of $\mathcal{U}$, and $L_2$ covering number of $\mathcal{W}$. Both these results are standard. From Lemma 6 in~\cite{sim_sham} we have 
\begin{equation}
\label{eqn:cov_u}
\mathcal{N}_{\infty}(\epsilon,\mathcal{G})\leq \log\left(\frac{1}{\epsilon}\right)+\frac{2s\sqrt{\log(2d)}}{\epsilon}. 
\end{equation}
Since, $\|w\|_1\leq \sqrt{s},\|x\|_{\infty}\leq \tilde{O}(\sqrt{\log(2d)})$, we can use Theorem 3 in~\cite{zhang2002covering},  to conclude that w.h.p.
\begin{equation}
\label{eqn:cov_w}
\log\mathcal{N}_2(\mathcal{W},\epsilon,n)\leq \frac{s\log^2(2d+1)}{\epsilon^2}.
\end{equation}
It is not hard to see that 
\begin{align}
\log \cN_2\left(\cF,\epsilon,n\right)&\leq \log\cN_2\left(\cW,\frac{\epsilon}{2\sqrt{2}},n\right)+\log\cN_{\infty}\left(\cG,\frac{\epsilon}{2\sqrt{2}}\right)\\
&= \tilde{O}\left(\frac{s\log^2(2d+1)}{\epsilon^2}\right)
\end{align}
Using Lemma A.1 in ~\cite{srebro2010smoothness} we can bound the worst case Rademacher complexity of $\cH$ by 
\begin{equation*}
\hat{R}_n(\cH)\leq \tilde{O}\left(\sqrt{\frac{s\log^2(2d+1)}{n}}\right)
\end{equation*}
 Finally applying Theorem 1 in~\cite{srebro2010smoothness} we get with probability at least $1-\delta$
\begin{equation*}
\err(h_T)\leq \widehat{\err}(h_T)+\tilde{O}\left(\sqrt{\widehat{\err}(h_T)}\sqrt{\frac{s\log^2(2d+1)}{n}}\right).
\end{equation*}
\end{proof}
\section{Proof of Theorem~\eqref{thm:hsimm}}
\begin{proof}
From Theorem~\eqref{thm:hsim1} we know that 
\begin{equation}
\label{eqn:one}
\mathcal{E}(\hat{h})=\err(\hat{h})-\err(h_{*})\leq \tilde{O}\left(\frac{(s+k)\log(2d)}{\theta}\sqrt{\frac{s}{n}}+\frac{1}{\sqrt{\theta}}\left(\frac{s}{n}\right)^\frac{1}{4}\sqrt{(s+k)\log(2d)}\right)
\end{equation}
Using Lemma~\ref{thm:large_dev} we can say that with probability at least $1-\delta$
\begin{align}
\label{eqn:two}
\widehat{\err}(\hat{h})&=\err(\hat{h})+\sqrt{\frac{s\log^2(2d+1)}{n}}=\err(h_{\star})+\err(\hat{h})-\err(h_{\star})+\sqrt{\frac{s\log^2(2d+1)}{n}}\\
	&=\err(h_{\star})+\tilde{O}\left(\frac{(s+k)\log(2d)}{\theta}\sqrt{\frac{s}{n}}+\frac{1}{\sqrt{\theta}}\left(\frac{s}{n}\right)^\frac{1}{4}\sqrt{(s+k)\log(2d)}\right)+\sqrt{\frac{s\log^2(2d+1)}{n}}.
\end{align}
Now consider $\hat{h}_{T}$ obtained by running \isilo~ for $T$ iterations, when initialized with $\hat{w},\hat{g}$ obtained by running \silo~ first on the data. Since $\hat{h}_T$ is chosen by using a held-out validation set as the iterate corresponding to the smallest validation error, we can claim via Hoeffding inequality that the empirical error of $\hat{h}_T$ cannot be too much larger than that of $\hat{h}$ (for otherwise $\hat{h}_T$ will not be the iterate with the smallest validation error). Precisely, if the validation set is of size $n$, then with high probability
\begin{equation}
\widehat{\err}(\hat{h}_T)\leq \widehat{\err}(\hat{h})+\tilde{O}\left(\frac{1}{\sqrt{n}}\right).
\end{equation}
Summing up Equations~\eqref{eqn:one} and~\eqref{eqn:two} we get 
\begin{equation}
\widehat{\err}(\hat{h}_T)\leq \err(h_{*})+\tilde{O}\left(\frac{(s+k)\log(2d)}{\theta}\sqrt{\frac{s}{n}}+\frac{1}{\sqrt{\theta}}\left(\frac{s}{n}\right)^\frac{1}{4}\sqrt{(s+k)\log(2d)}+\sqrt{\frac{s\log^2(2d+1)}{n}}+\sqrt{\frac{1}{n}}\right)
\end{equation}
Now using Theorem~\eqref{thm:large_dev} to upper bound $\err(\hat{h}_T)$ in terms of $\widehat{\err}(\hat{h}_T)$, and combining it with the above bound we get the desired result. The same arguments apply even to the \cisilo~ algorithm.
\end{proof}

\section{Additional Experimental Results}
Here we report results on other high dimensional datasets. Figure \ref{suppress} again shows the advantage of the calibrated, and iterative method ciSILO. 
\begin{figure}
\centering
\includegraphics[width= 140mm, height = 45mm]{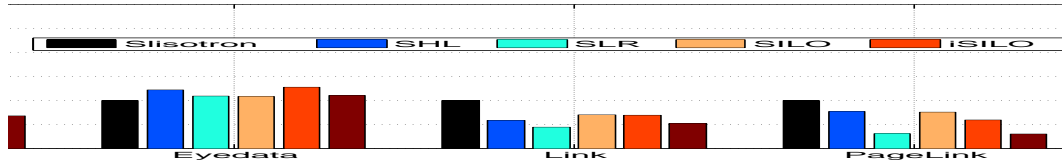}
\caption{Comparison of different methods over different datasets. The results are normalized so that the Slisotron has error $=1$}
\label{suppress}
\end{figure}
Table \ref{tabdet} has the details of the datasets in Figure \ref{suppress}
\begin{table}
\centering
\begin{tabular}{  |l c | c | c ||   }
 \hline
 \textbf{Dataset} & \textbf{n} & \textbf{d}  \\
\hline
 Leukamia & 44  &  7129  \\
  Eyedata & 120  & 200  \\
   Link  & 526 & 1840  \\
   Page$+$Link  &  526 & 4840 \\
   Gisette & 4200 &  5000 \\
   \hline
\end{tabular}
\caption{Dataset details}
\label{tabdet}
\end{table}

\end{document}